\documentclass[conference]{IEEEtran}
\IEEEoverridecommandlockouts

\usepackage{cite}
\usepackage{amsmath,amssymb,amsfonts}
\usepackage{algorithmic}
\usepackage{dsfont}
\usepackage{graphicx}
\usepackage{textcomp}
\usepackage{xcolor}
\usepackage[utf8]{inputenc}
\usepackage[small]{caption}
\usepackage{float}
\usepackage{url}
\usepackage{graphicx}
\usepackage{subcaption}
\usepackage{amssymb}
\usepackage{amsmath}
\usepackage{amsthm}
\usepackage{afterpage}
\usepackage{booktabs}
\usepackage[switch]{lineno}
\usepackage{color}
\usepackage{array,multirow}
\newcommand{\bulle}{\textcolor[rgb]{0.00,0.00,1.00}{\bullet}}
\usepackage[titlenotnumbered,ruled,vlined,noend,linesnumbered]{algorithm2e}
\usepackage{times}
\usepackage[
  colorlinks=true,
  citecolor=blue,
  linkcolor=blue,
  urlcolor=blue
]{hyperref}

\newtheorem{theorem}{Theorem}

\newtheorem{proposition}[theorem]{Proposition}

\newtheorem{definition}{Definition}
\newcommand{\X}{\boldsymbol{X}}
\newcommand{\x}{\mathbf{x}}

\newcommand{\R}{\mathbb{R}}

\newcommand{\cov}{\text{Cov}}

\def\BibTeX{{\rm B\kern-.05em{\sc i\kern-.025em b}\kern-.08em
    T\kern-.1667em\lower.7ex\hbox{E}\kern-.125emX}}

\begin{document}

\title{Optimizing Multidimensional Scaling in Gini Metric Spaces
}

\author{\IEEEauthorblockN{1\textsuperscript{st} Cassandra Mussard}
\IEEEauthorblockA{
\textit{CESBIO, Univ. Toulouse}\\
Toulouse, France \\
009-0008-8679-0088}
\and
\IEEEauthorblockN{2\textsuperscript{nd} Stéphane Mussard}
\IEEEauthorblockA{\textit{CHROME, Univ. Nîmes} \\
\textit{UM6P, AIRESS}\\
Nîmes, France \\
0000-0003-2331-630X}
}
\maketitle

\begin{abstract}
 The Gini Multidimensional Scaling (Gini MDS) framework extends the Euclidean multidimensional scaling. We introduce a Gini pseudo-distance based on values and their ranks that depends on a fine-tunable hyperparameter. This pseudo-distance allows flexible exploration of latent configurations, enabling embeddings that best match observed dissimilarities. The Gini MDS is shown to be robust to noise and outliers, making it well-suited for real-world applications. We provide experiments on 16 UCI datasets with outliers and on MNIST images with noise to show that the Gini MDS outperforms the Euclidean MDS on noisy data. Finally, a tensor-based implementation in \texttt{PyTorch} provides GPU acceleration and efficient computation compared to the standard MDS of the \texttt{sklearn} library.  
\end{abstract}


\section{Introduction}
Multidimensional scaling (MDS) is a multivariate analysis tool of dimensionality reduction aimed at generating points and visualizing the structure of complex data through a distance matrix computed in a low-dimensional geometric space. Given a matrix of distances (or dissimilarities) between points, MDS fits a Euclidean space in which the distances between points approximate the true distances as closely as possible. This method is generally used when information about the distances between instances is available only, whereas the features and their correlations remain unknown. This is particularly relevant in many fields, where datasets are often small, allowing to infer latent structures that are not directly observable. By transforming dissimilarities or distances between instances, MDS provides a new metric space (in general in two or three dimensions) in which correlations and classifications may be performed. 

The literature brings out a wide quantity of applications, ranging from operational research to computer science. For instance, \cite{Safizadeh1996}
demonstrate the use of non-metric multidimensional scaling in the design of facility layouts. They show that different distance metrics can be employed over many instances without restrictions. They also show that MDS may offer a low computation cost and an intuitive visualization of the generated data into two or three dimensions. \cite{GiladiSpectorRaveh1984} study the computational performances of the MDS under hardware attributes such as the importance of IO rate, cache size, multiprocessing, and CPU speed. 
Another branch of the literature is more related to the study of correlations and statistical inference. \cite{Hooley1998} use MDS to model product positions, precisely to analyze perceptions of the similarity between cigarette brands. The MDS provides a perceptual (latent) space that reflects their relative positioning   (brands differentiation), \textit{i.e.} the pattern of the market structure under investigation. \cite{Abe1998} investigates the theoretical foundation of maximum likelihood estimation in non-metric MDS, focusing on error modeling and parameter estimation. A new error structure is proposed, \textit{i.e.} a multiplicative disturbance associated to distances, to improve the reliability of estimation with asymptotic confidence regions.

Recent MDS papers are related to the concept of metric distances used in algorithms under the prism of machine learning. For example, \cite{DuLi2024} employ the K-nearest neighbors, relying on the classical Euclidean distance, in order to fill missing values in the matrix of dissimilarity (see also, \cite{BLUMBERG2025104234}).  In the same vein, \cite{GachkoobanAlizadehShakeri2024} propose to estimate the coordinates of the instances under the assumption that each instance depends only on its own position, so that its distances with other points are issued from its neighbors. On the other hand, \cite{DelicadoPachonGarcia2024} suggest algorithms to handle several thousand of points. Their idea is to use different batches of distances, so that the classical MDS can be applied iteratively on these batches. Again, for large-scale data, \cite{CanzarEtAl2024} propose a neural network–based method that scales metric MDS to millions of instances, with a particular focus on single-cell RNA sequencing. This approach outperforms conventional MDS in speed over big datasets.

The Gini MDS framework extends classical MDS with different contributions. 

(i) This pseudo-distance relies on couples of values and ranks that can be designed into a proper distance metric with an appropriate normalization. In this respect, it captures correlations that depend on ranks and values in the data. 

(ii) The Gini pseudo-distance includes a tunable hyperparameter, which enables the exploration of a wide family of generated points (embeddings) and facilitates the selection of the best latent space that fits the observed dissimilarities. As in \cite{CanzarEtAl2024,RamSabach2024,Trosset2024,DzemydaSabaliauskas2024} a stress minimization approach is performed. The difference between the Gini pseudo-distances of the original instances and the Gini pseudo-distances of the embeddings (in the latent space) is minimized to solve for the hyperparameter. 

(iii) The Gini pseudo-distance is robust to noise due to its structure relying on both ranks and values. 

(iv) A tensor-based implementation of the Gini MDS is proposed using \texttt{PyTorch} for GPU usage and efficient computation, and is subsequently applied to 16 UCI datasets, MNIST data, and heavy-tailed distributions to illustrate its behavior with and without contamination. 

The outline of the paper is the following. Section \ref{literature} presents the state of the art about Gini distances. Section \ref{MDS} extends the Euclidean (classical) MDS to optimized Gini MDS based on stress metrics. Section \ref{xp} deals with classification tasks on 16 UCI datasets and MNIST images with both normal conditions and contaminated data. Section \ref{conslusion} closes the paper.

\section{State of the art: Gini metrics}\label{literature}

The Gini regression  \cite{Olkin1991} and the generalized Gini regression  \cite{yitzhaki2013gini} rely on the minimization of the Gini covariance used as a distance function. Consider a model $Y = \alpha + \beta X + \varepsilon$ with $(X,Y)$ following some bivariate distributions, with $\hat{\alpha} = \bar{y} - \hat{\beta}\bar{x}$ the intercept of the model, the residuals $e = y - \hat{y}$, and $\overline{F}_e$ the survival distribution function of the residuals $\hat{\varepsilon}= e$. The estimator of the parametric Gini regression is:
\begin{equation}\notag
\hat{\beta}^{P} = \arg\min_{\nu}\cov(e,-\overline{F}_{e}^{\nu-1}(e)) 
\end{equation}
When $\nu=2$, the minimization of $\cov(e,-\overline{F}_{e}^{\nu-1}(e))$ corresponds to the minimization of the distance $\mathbb{E}{|\varepsilon_1-\varepsilon_2|}$ (with $\varepsilon_1,\varepsilon_2$ two copies of $\varepsilon$), that is the minimization of the Gini index of the residuals. This process is equivalent to the robust median regression \cite{Olkin1991}. When $\nu - 1 > 1$, more weight is applied to the lower part of the distribution of $e$, whereas if $\nu - 1 < 1$, more weight is assigned to the upper part of the distribution of $e$, which allows one to control the sensitivity of the method to outliers.

\cite{Sang2024} propose to employ a Gini distance correlation for classification purpose. If $Y$ is a categorical variable with $l$ labels $Y = (L_1,\cdots,L_l)$, with  $F_{j}$ the cdf of $X_j$ (group $j$) and $F_{j|k}$ the conditional cdf of $X_j$ given $Y=L_k$, the Gini distance correlation (univariate) can be expressed in the following way:
$$
\text{gCov}(X_j,Y) = \sum_{k=1}^K \mathbb{P}(Y = L_k) \int_{\mathbb R} (F_{j|k}(x) - F_{j}(x)) dx 
$$
The gCov distance can be extended to a multivariate setting for classification purposes. \cite{Sang2024} apply this distance to the linear discriminant analysis and binary logistic regression in a ultrahigh dimensional setting, showing the better performance of gCov compared to traditional $R^2$ metric issued from the variance analysis especially in case of heavy-tailed error distributions. The gCov distance relies on two properties, the sure independence screening property and the ranking consistency property, however it is not a distance metric. 

\cite{Mussard2025} propose a Gini \textit{prametric} for supervised and non-supervised classification tasks, \textit{i.e.} for K-nearest neighbors and K-means algorithms respectively. A prametric is not a distance since it may fail to satisfy symmetry and the triangle inequality. The Gini prametric is symmetric, non-negative, and it is null if $x=y$ or if the cdfs of $x$ and $y$ are equal. 

Compared to the previous Gini metrics, our proposed generalized Gini pseudo-distance takes benefit from the triangle inequality property. 
Although this property can be viewed as non-essential in many spaces, it remains one of the building-block property of Euclidean metric spaces. In particular, projecting points into a reduced Euclidean space can exploit the triangle inequality to preserve their relative proximity, as shown in the robust MDS proposed by \cite{Blouvshtein2018} where the triangle is checked for all sets of 3 instances. Accordingly, generating MDS embeddings by means of the generalized Gini pseudo-distance may help to preserve the local structure of the data. Furthermore, as the Gini distance is a robust statistic, it allows to control for the noise and outliers in the data. Consequently the Gini MDS falls into the category of robust MDS, such as Huber MDS \cite{Forero2012}, outlier detection MDS \cite{Li2023}, or robust MDS inspired from robust PCA \cite{DengWang2025}.

\section{Generalized Gini pseudo-distances}\label{distance}

In this section, we begin with the definition of the Gini pseudo-distance. Then, we present the generalized Gini pseudo-distance that accounts for different weights to be put on the distributions of points. Finally, a comparison is made with other Gini distances.  

\subsection{Gini pseudo-distance}

\textit{Notations.} Let $\mathbb E$ be a nonempty set such that $\mathbb E \subseteq \R$ and $\mathbb E^d $ its $d$-dimensional representation with $d\in \mathbb N \setminus \{0\}$, $\mathbb N$ be the set of integers. Let $0_d$ be the $d$-dimensional vector of zeros and $\mathds{1}_d$ the $d$-dimensional vector of ones. The cardinal of set $A$ is denoted $\#\{A\}$, and the set of all positive integers including $d\in \mathbb N$ is $[d]:=\{1,\cdots,d\}$. Let $\X \equiv [x_{ij}]$ being a matrix of size $n\times d$ with $\x_i$ and $\x_j$ representing respectively the $i$th row and the $j$th column of $\X$, $\bar{\x}_i$ being the arithmetic mean of $\x_i$.  

The Minkowski $p$-norm is, for all $x \in \mathbb E^d$ with $p\geq 1$,
$$
\lVert x \rVert_p \equiv d_p(x,0_d) := \sqrt[p]{ \sum\limits_{j=1}^d |x_{j}|^p }
$$
For $p=2$, the well-known Euclidean norm is given by, as a distance from $0_d$,
$$
\lVert x \rVert_2 = \lVert x - 0_d \rVert_2 = d_2(x,0_d) = \sqrt{\sum\limits_{j=1}^d (x_j - 0)^2}
$$
Translating this into a Gini framework implies that information on gaps and on rank gaps should be taken into account. 

\begin{definition}  
Let $u\in \mathbb E^d$ and let $R:\mathbb E \to \mathbb N\setminus \{0\}$ be the rank function that yields the position of each element $u_i$ of $u$ such that, $R(u_i) = 1 + \sum_{j=1}^{d} \mathbb{I}(u_j < u_i)$, with $\mathbb{I}(u_j < u_i)=1$ if $u_j < u_i$, $0$ otherwise. The Gini norm is defined as, for all $x\in \mathbb E^d$,
\begin{equation}\label{Gini-norm}
\lVert x \rVert_G := \sum\limits_{j=1}^d (x_j-0)(R(x_{j}) - R(0))
\end{equation}
\end{definition}
Ranks of ties can be estimated as midpoints, this strategy is used to avoid biases in the case of the Gini coefficient estimation, see \cite{yitzhaki2013gini}:
\begin{equation}\label{rank-ties}
R(x_j) = \frac{1 + \sum_{i=1}^{d} \mathbb{I}(x_i \leq x_j)}{\#\{x_i | x_i = x_j\}}
\end{equation}
This implies $R(0)=\frac{d+1}{2}$, and  Eq.\eqref{Gini-norm} becomes:
$$
\lVert x \rVert_G = \sum\limits_{j=1}^d x_j\Big(R(x_{j}) - \frac{d+1}{2}\Big)
$$
In this respect, we can show that $\lVert x \rVert_G$ is a seminorm.

\begin{proposition}\label{proposition-distance} \emph{\textbf{(Gini seminorm)}}
\newline Let $x,y \in \mathbb E^d$ and $\lambda \in \mathbb R$:
\newline $(\imath)$ $\|\lambda x \|_G = \lvert \lambda \rvert \| x \|_G$ \textcolor{blue}{\emph{(Absolute homogeneity)}}
\newline $(\imath\imath)$ $ \|x+y\|_G \leq  \|x\|_G + \|y \|_G$ \textcolor{blue}{\emph{(Triangle inequality)}}
\end{proposition}
\begin{proof}
See the Appendix.
\end{proof}

Contrary to the the Minkowski $p$-norm for which the norm is null if and only if $x = 0_d$, two cases may arise in the case of the Gini seminorm. Let the rank vector $R(x):=(R(x_1),\ldots,R(x_d))$:
$$
\lVert x \rVert_G =
\left \{
\begin{array}{ll}
   0,  & \text{if } x = 0_d \\
   0,  & \text{if } R(x) = c\mathds{1}_d, \ \forall c \in \R \setminus \{0\} \end{array}
\right.
$$
Consequently, the Gini seminorm cannot be a norm since the property of point separation \textit{i.e.} $\lVert x \rVert_G = 0 \ \Longrightarrow \ x = 0_d$ is not respected. On the basis of the Gini seminorm, a Gini pseudo-distance can be defined accordingly. 

\begin{definition}  \textbf{\emph{(Gini pseudo-distance)}}
\newline Given the rank function defined in Eq.\eqref{rank-ties}, the Gini pseudo-distance is defined as, for all $x,y\in \mathbb E^d$,
\begin{equation}\notag
\mathcal{D}_G(x,y) := \lVert x-y \rVert_G = \sum\limits_{j=1}^d (x_j-y_j)(R(x_{j}-y_j) - R(0))
\end{equation}
\end{definition}

Consequently, the Gini pseudo-distance function is null if $x = y$, a standard property of distance functions. This condition implies that $R(x) = R(y)$, however in cases of egalitarian distributions (but not identical $x\neq y$), $x = \lambda y$ such that $\lambda\neq 0$ and $x = c\mathds{1}_d$ with $c\neq 0$, then $R(x - y) = R(0_d)$, and therefore $\mathcal{D}_G(x,y)=0$. This means that the Gini pseudo-distance is sensitive both to value and rank differences. 

Let us take a simple example. Let $u = (10, 1200)$ and $v = (10, 12)$, therefore $R(u-v) = (1, 2)$. Then,
\begin{align}
    \mathcal{D}_G(u,v) &= (10-10)(1-3/2) + (1200-12)(2-3/2) \notag \\
    &= 0 + 1188 \times 0.5 \notag \\
    &= 594 \notag
\end{align}
The Gini distance function allows important gaps such as $1200-12$ to be attenuated by the rank difference valued to be $2-3/2 = 0.5$. In cases of measurement errors, for instance if the value $1200$ were wrongly assigned to $u_2$ instead of $12$, the weight $0.5$ would contribute to shrink the gap. The Euclidean distance remains sensitive to gaps only, even more if it relies on gaps including outlying observations. This idea of using ranks to obtain robust statistics was first proposed by \cite{Spearman1904} to measure rank correlations (robust to outliers) and later by \cite{Durbin1954} who proposed to use ranks as instruments in regressions.


\begin{proposition}\label{proposition-distance-2} \emph{\textbf{(Gini pseudo-distance properties)}}
\newline Let $x,y,z \in \mathbb E^d$, $ c \in \R \setminus \{0\}$ and $ \lambda \in \R \setminus \{0,1\}$:
\newline $(\imath)$ $\mathcal{D}_G(x,y) = 0$ if $x=y$ \textcolor{blue}{\emph{(Null)}}
\newline $(\imath\imath)$ $\mathcal{D}_G(x,y) = 0$ if $x=c\mathds{1}_d = \lambda y$ \textcolor{blue}{\emph{(Egalitarian distributions)}}
\newline $(\imath\imath\imath)$ $\mathcal{D}_G(x,y) = \mathcal{D}_G(y,x) $ \textcolor{blue}{\emph{(Symmetry)}}
\newline $(\imath v)$ 
$\mathcal{D}_G(x,y) \leq \mathcal{D}_G(x,z)+\mathcal{D}_G(z,y)$ \textcolor{blue}{\emph{(Triangle inequality)}}
\newline $(v)$ $\mathcal{D}_G(x, y) \geq 0$ \textcolor{blue}{\emph{(Non-negativity)}}
\newline $(v\imath)$ The Gini pseudo-metric space $(\mathbb E^d, \mathcal{D}_G)$ can be transformed into a metric space.
\end{proposition}

\begin{proof} 
See the Appendix.\end{proof}

Property ($\imath \imath$) indicates that the Null property is not always issued from the condition $x=y$. Indeed, if $R(x_j-y_j)=R(0)$ for all $j\in [d]$ the Gini pseudo-distance is null. In other terms this condition rewrites $x=c\mathds{1}_d = \lambda y$ with $\lambda \neq 0,1$, \textit{i.e.} the distributions $x,y$ are egalitarian (but not identical). It is noteworthy that property ($v \imath$) is not very demanding and is in line with ($\imath \imath$). Indeed, a Gini metric space can be obtained by centering the data. Let the vector $\tilde x$ be centered such that $\tilde x = x - \bar{x}$. Therefore, we can define a set $\mathbb G^{d}$ such that:
$$
\mathbb G^{d} := \{ \tilde x \in \mathbb E^{d} : \tilde x = x - \bar{x} \}
$$
The space $(\mathbb G^{d}, \mathcal{D}_G)$ defines a Gini metric space that relies on centered data. Because pre-processing is common practice before running algorithms, the metric space $\mathbb G^{d}$ can be employed to deal with centered data. 

Although the properties defined in Proposition \ref{proposition-distance-2} are appealing, a more robust version of the Gini pseudo-distance is proposed in Section \ref{generalized}. 

\subsection{Generalized Gini pseudo-distance}\label{generalized}

\cite{Schechtman1987,schechtman2003family,yitzhaki2013gini} introduce a generalized Gini Mean Difference as a robust statistic compared to the usual covariance that is sensitive to outliers:
\begin{equation}\label{GMD_0}
GMD_\nu(X,Y) = -{2\nu}\ \cov(X,\overline{F}_Y(Y)^{\nu-1}) ,  \ \nu > 1,
\end{equation}
with $\overline{F}_Y$ the survival distribution function of $Y$ ($\overline{F}_Y=1-F_Y$). The Gini covariance operator $\cov(X,\overline{F}_Y(Y)^{\nu-1})$ is a robust version of the standard covariance operator \cite{yitzhaki2013gini}. Indeed, if $\nu=2$ more weight is put on the median of $Y$; if $\nu \in (1,2)$ more weight is put on the upper part of $Y$ (for instance if some outliers appear in the lower part of $Y$), and finally if $\nu >2$ more weight is put on the lower part of the distribution of $Y$ (which is a good strategy if outliers occur at the upper part of the distribution of $Y$). 

A natural way to build a generalized Gini pseudo-distance is to take benefit from the hyperparameter $\nu$ in order to focus on some particular parts of the distribution of $X-Y$. 
To put more weight on the tails of the distribution of $X-Y$, the hyperparameter $\nu$ can be employed as follows (for $\nu > 1$):
\begin{equation}\notag
\mathcal D_{G,\nu}(x,y) := -d\sum\limits_{j=1}^d (x_j - y_j)\Big(\overline{F}^{\nu-1}(x_{j} - y_{j}) - \overline{F}^{\nu-1}(0)\Big)
\end{equation}
The metric $\mathcal D_{G,\nu}(x,y)$ is not a pseudo-distance since it is not symmetric. Applying a symmetrization of $\mathcal D_{G,\nu}(x,y)$ yields a generalized Gini pseudo-distance:
\begin{equation}\notag
\overline{\mathcal D}_{G,\nu}(x,y) := \frac{1}{2}\mathcal D_{G,\nu}(x,y) + \frac{1}{2}\mathcal D_{G,\nu}(y,x), \ \nu > 1
\end{equation}

\begin{proposition}\label{proposition-distance-generalized}
Let $\overline{F}$ be symmetric and $\nu > 1$, then $\overline{\mathcal D}_{G,\nu}(x,y)$ is a pseudo-distance that satisfies properties $(\imath)-(v\imath)$ of Proposition
\ref{proposition-distance-2}. 
\end{proposition}

\begin{proof} 
See the Appendix.
\end{proof}

\section{Gini multidimensional scaling}\label{MDS}

We first expose the classical (Euclidean) MDS, then we propose the optimized Gini MDS algorithm, issued from the fine-tuned hyperparameter $\nu$ of the generalized Gini pseudo-distance.  

\subsection{Classical MDS}

Given a symmetric distance (dissimilarity) matrix \( \boldsymbol{D} = [d_{ij}] \in \mathbb{R}^{n \times n} \), the aim of MDS is to find a set of points \( x_1, \dots, x_n \in \mathbb{R}^p \) such that the pairwise Euclidean distances in the latent space (embeddings) approximate the original distances $
d_{ij} := \|x_i - x_j\|_2$ between all pairs of points $(i,j)$. Let $ \boldsymbol{D}^{2} \in \mathbb{R}^{n \times n} $ be the matrix of squared distances $d_{ij}^2$. Let $\boldsymbol{H}$ be the centering matrix $\boldsymbol{H} = I - \frac{1}{n} \mathds{1}_n \mathds{1}_n^\top$, then the Gram matrix $\boldsymbol{B} \in \mathbb{R}^{n \times n}$, which contains inner products between the centered data points, is obtained by double-centering:
\[
\boldsymbol{B} = -\frac{1}{2} \boldsymbol{H} \boldsymbol{D}^{2} \boldsymbol{H}
\]
The eigendecomposition of $\boldsymbol{B}$ yields $\boldsymbol{B} = V \Lambda V^\top$. Let $\Lambda_p \in \mathbb{R}^{p \times p}$ contains the top $p$ eigenvalues of $\boldsymbol{B}$, and $V_p \in \mathbb{R}^{n \times p}$ the corresponding eigenvectors. Thereby, the matrix $\boldsymbol{\tilde X} \in \mathbb{R}^{n \times p}$ of the embedded points in $\mathbb{R}^p$ is given by:
\[
\boldsymbol{\tilde X} = V_p \Lambda_p^{1/2}
\]
This yields a latent space equivalent to that issued from the principal component analysis, except that the space is induced by a matrix of distances rather than features that are not directly available.

\subsection{Optimized Gini MDS}

The Gini MDS consists of fitting the latent space by assuming that the distance matrix $\boldsymbol{D}$, constructed from pairwise generalized Gini distances, is induced by a Gini (pseudo-)metric space. In order to fit the best latent space, the hyperparameter is obtained by a simple minimization. The most common statistic is the Kruskal stress \cite{Kruskal1964,deLeeuwMair2009}, which can be used for any given distance $d_{ij}$:
    \[
    \text{KS} := \sqrt{ \frac{ \sum_{i<j} \left( \|\tilde x_i - \tilde x_j\|_2 - d_{ij} \right)^2 }{ \sum_{i<j} d_{ij}^2 } }
    \]
In Algorithm \ref{algo-GiniMDS} below, the optimized Gini MDS is displayed with batches (folds) in cases where the matrix of distances is issued from a large sample. For each value of the hyperparameter $\nu$, the Kruskal stress is performed in order to find the latent space that best fits the data, that is, the latent space for which the gap between distances in the latent space and initial distances is minimized.

\begin{algorithm}
\caption{
\textbf{Optimized Gini MDS}
}\label{algo-GiniMDS}
\KwIn{Training data $\X$}
mean\_stress = [ \ ] \\
\For{$\nu \in [1.1;5]$}{
    stress\_score = [ \ ]\; 
    \For{\emph{fold} in \emph{Kfolds}}{
        Compute Gini pseudo-distances $\overline{\mathcal D}_{G,\nu}(\X)$[fold] $\equiv [\overline{\mathcal D}_{G,\nu}(x_i,x_j)]$ \; 
        Run MDS on $\overline{\mathcal D}_{G,\nu}(\X)$[fold] \; Deduce latent space $\tilde{\X}$[fold]\; 
        Compute Euclidean distances on $\tilde{\X}$[fold]: $\|\tilde{x}_i - \tilde{x}_j\|_2$\; 
        Evaluate Kruskal’s Stress:
        \[
            \text{KS}(\nu) := \sqrt{ \frac{\sum_{i<j} \left( \|\tilde{x}_i - \tilde{x}_j\|_2 - \overline{\mathcal D}_{G,\nu}(x_i,x_j) \right)^2}{\sum_{i<j} \overline{\mathcal D}_{G,\nu}(x_i,x_j)^2} }
        \]
        stress\_score[fold] $\leftarrow$ KS$(\nu)$\;
    }
    mean\_stress[$\nu$] $\leftarrow$ mean(stress\_score) }
\Return $\nu^* = \arg\min$ mean\_stress, $\tilde \X$
\end{algorithm}

The optimal solution $\tilde{\X}$ obtained for $\nu^*$ can be compared with standard MDS, as shown in Figure \ref{fig:Gini-Euclidean} on the cars dataset (with two outliers: Ferrari Enzo at the top and Smart on the right), highlighting a better preservation of pairwise distances with outliers (we will come back to this point in Section \ref{sec:simul}).

\begin{figure}[h!]
    \centering
    \begin{subfigure}{0.23\textwidth}
        \centering
        \includegraphics[width=\linewidth]{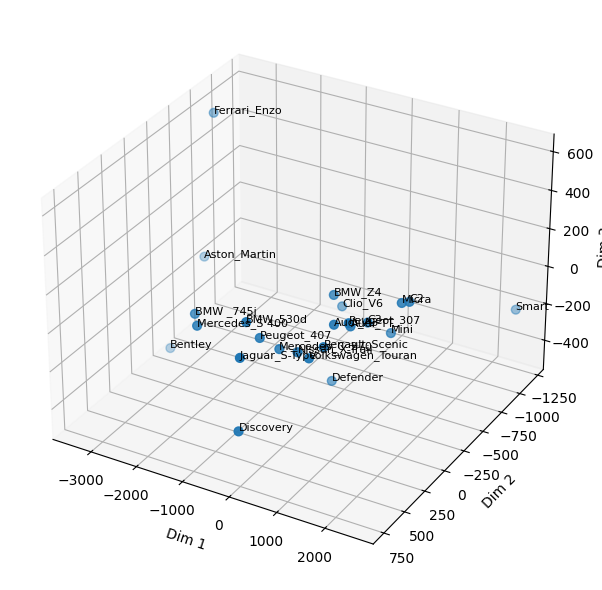}
        \caption{Euclidean MDS}
        \label{fig:euclid}
    \end{subfigure}
    \hfill
    \begin{subfigure}{0.23\textwidth}
        \centering
       \includegraphics[width=\linewidth]{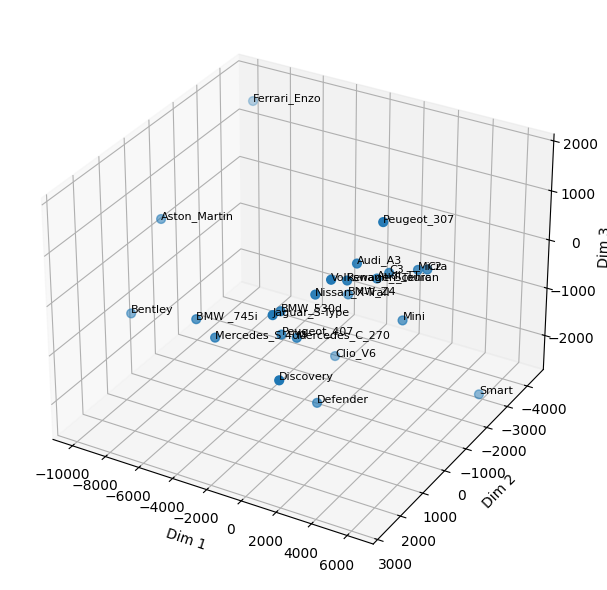}
        \caption{Gini MDS $\nu^*=9.4$}
        \label{fig:gini}
    \end{subfigure}
    \caption{Comparison of Euclidean and Gini MDS embeddings}
    \label{fig:Gini-Euclidean}
\end{figure}

However since the classical MDS is not optimized, we compare the Gini MDS to three optimized Euclidean MDS. The first one is a robust MDS derived from \cite{huber1964robust}. The errors, $e_{ij} := \| x_i - x_j \|_2 - \| \tilde x_i - \tilde x_j \|_2$ are penalized with the following loss function,

\begin{equation*}
\begin{split}
\mathcal{L}_{\text{Huber}} &=
\sum_{i < j} \rho_\delta \bigl( e_{ij} \bigr) \\
&\quad \text{where } \rho_\delta(e_{ij}) =
\begin{cases}
\frac{1}{2} e_{ij}^2, & \text{if } |e_{ij}| \le \delta, \\
\delta (|e_{ij}| - \tfrac{1}{2}\delta), & \text{otherwise.}
\end{cases}
\end{split}
\end{equation*}
Iterating over many values of $\delta$, the Huber loss penalizes small errors quadratically and large errors linearly, therefore  it is more robust to outliers than the Euclidean MDS. Second, the Sammon loss is used \cite{sammon1969mapping},
\[
\mathcal L_{\text{Sammon}} = 
\frac{1}{\sum_{i<j} \| x_i - x_j \|_2}
\sum_{i<j} \frac{e_{ij}^2}{\| x_i - x_j \|_2}
\]
It is close to Kruskal stress, but it imposes more weight to preserving small distances, consequently this preserves the local structure in $\tilde{\mathbf{X}}$. Third, SMACOF loss \cite{deleeuw1977applications},
\[
\mathcal L_{\text{SMACOF}}(\tilde \X) = 
\sum_{i < j} w_{ij} e_{ij}^2
\]
where $w_{ij}$ are nonnegative weights. These weights are often set to $w_{ij}=1$ so that the least-squares loss is minimized between
pairwise distances in $\X$ and in $\tilde{\X}$.

\section{Experiments}\label{xp}

Three experiments assess Gini MDS for supervised classification: (1) 1D MDS on 16 outlier-contaminated UCI datasets; (2) noisy MNIST; (3) comparison with three nonlinear MDS on heavy-tailed distributions (GPU-enabled \texttt{PyTorch} code available at \href{https://github.com/bangtan66708/MDS_Gini_pseudo_metric/tree/main}{GitHub}).

\subsection{UCI datasets}

As in \cite{Mussard2025}, 16 UCI datasets are selected due to their differences in the number of variables (features) and outliers (see Table \ref{datasets}).\footnote{The isolation forest test has been employed to compute the number of outliers (with 1\% contamination).} The Gini MDS is compared with the four Euclidean ones, without noise in the data, and with 5\% contamination to test for the robustness of the MDS algorithms.

\begin{table}[H]
\centering
\scriptsize
\caption{Description of the UCI datasets: \#E = number of instances, \#F = number of features, \#C = number of classes}
\begin{center}
\begin{tabular}{c|c|c |c| c| c|c} 
 \hline
 Name & \#E & \#F & \#C & Data type & \# Outliers & Biblio   \\ [0.5ex] 
 \hline\hline
 Australian  & 690 & 42 & 2 & Real &7 & \citenum{Quinlan1987}   \\
 \hline
 Banknote & 1372 & 4 & 2 & Real &14 & \citenum{Banknote}  \\
 \hline
 Breast-Cancer & 569 & 30 & 2 &Real, numeric & 6 &\citenum{Wolberg1993} \\
 \hline
 Ionosphere & 351 & 34 & 2 & Real &4 & \citenum{Ionosphere} \\
 \hline
 Iris & 150 & 4 & 3 & Real, numeric & 2 & \citenum{Iris} \\ 
 \hline
 German  & 1000 & 20 & 2 & Integer, Categorical & 10 & \citenum{German} \\
 \hline
 Glass  & 214 & 9 & 6 & Real & 3 & \citenum{german1987glass} \\
 \hline
Heart  & 303 & 14 & 2 & Integer &4 & \citenum{Janosi1989heart}  \\
 \hline
 QSAR  & 1055 & 41 & 2 & Real &8 & \citenum{mansouri2013qsar} \\
\hline
 Sonar  & 208 & 60 & 2 & Real, Categorical & 3 &\citenum{sejnowski1988sonar} \\
\hline
 Vehicle  & 946 & 18 & 4 & Integer & 9 & \citenum{vehicle} \\
 \hline
 Wine  & 178 &13& 3 & Real, numeric &2 & \citenum{wine} \\
 \hline
  Balance  & 625 &4& 3 & Integer & 7 & \citenum{balance} \\
 \hline
   Haberman  & 306 &3& 2 & Integer & 4 &\citenum{haberman} \\
 \hline  Wholesale  & 440 &7& 2 & Integer & 5 &\citenum{wholesale}  \\
  \hline  Indian Liver  & 346 &6& 2 & Real, Integer & 6 & \citenum{Ramana2022} \\
 \hline
\end{tabular}
\end{center}
\label{datasets}
\end{table}

\subsubsection{Metrics of similarity}

To evaluate and compare the MDS methods, similarity metrics are employed to assess whether the embeddings $\tilde{\X}$ resemble the original data $\X$. Three different metrics are used. First, the trustworthiness metric is used with $k = 5$ nearest neighbors to assess the local structure preservation of the MDS embedding. It penalizes points that appear among the $k$ nearest neighbors of point $i$ in the embedded space $\tilde{\X}$ but not in the original space $\X$, based on their rank $R(\X[i,j])$ in the neighborhood ordering of $i$ in $\X$:
\[
T(k) = 1 - \frac{2}{n k (2n - 3k - 1)}
\sum_{i=1}^{n}
\sum_{j \in N_i} (R(\X[i,j]) - k)
\]
When $T(5)=1$, the MDS yields a perfect preservation of local neighborhoods (among 5 neighbors) in the embedded space $\tilde \X$, while a negative score represents a bad preservation of the distances between the 5 neighbor points in $\tilde \X$. Second, the nearest neighbors statistic yields the proportion of neighbors in the embedded space $\tilde \X$ that share the same label. Then, it provides how the MDS locally gathers similar points:\footnote{Contrary to trustworthiness in which false neighbors are penalized, the $k$ nearest neighbors metric ($NN(k)$) gives a simple proportion of true neighbors among $k$ nearest points (the $k$ nearest points in set $N_i$ excluding $i$ are identified with the Euclidean distance in $\tilde \X$).} 
\[
NN(k)
= \frac{1}{n k} 
\sum_{i=1}^{n} \sum_{j \in N_i} 
\mathbb{I}(y_j = y_i)
\]
In the following experiments, $k$ is set to 10 to test for a larger set of neighbors compared to trustworthiness. 
Third, the Silhouette score is used. For each point $i$, the Silhouette value $s_i$ compares the average distance to points with the same label ($d_i$) to the minimum average distance to points with a different label ($n_i$), with the comparison normalized by $\max(d_i, n_i)$:
\[
S = \frac{1}{n}\sum_{i=1}^{n} s_i, \text{ with } s_i = \frac{n_i - d_i}{\max\big(d_i,\, n_i\big)}
\]
While MDS aims to preserve pairwise distances, the available labels are employed for external evaluation to assess how well the embedding preserves distances among points with the same label. 


\subsubsection{Experiments without outliers}

The three metrics of similarity are computed to assess the quality of the Euclidean MDS on the 16 UCI datasets (Huber, Sammon and SMACOF). The optimized Gini MDS is compared with the three aforementioned Euclidean MDS variants. Table \ref{results-MDS} reports the number of times the Gini MDS is top-ranked relative to all Euclidean MDS techniques across the 16 UCI datasets.

\begin{table}[H]
\setlength{\tabcolsep}{2.5pt}
\footnotesize
\caption{\# Rank 1 occurrences on 16 datasets}
\label{results-MDS}
\begin{tabular}{lcccccccc}
\hline\hline
\multirow{2}{*}{\textbf{Components}}
 & \multicolumn{2}{c}{\textbf{Trust} $T(5)$}
 & \multicolumn{2}{c}{\textbf{Neighb.} $NN(10)$}
 & \multicolumn{2}{c}{\textbf{Silhouette} $S$} \\
\cline{2-7}
 & \textbf{MDS} & \textbf{Gini}
 & \textbf{MDS} & \textbf{Gini}
 & \textbf{MDS} & \textbf{Gini} \\
\hline
Component 1 & 13 & 3 & 10 & 6 & 9  & 7 \\
Component 2 & 13 & 3 & 10 & 6 & 10 & 6 \\
Component 3 & 14 & 2 & 8  & 8 & 7  & 9 \\
\hline
\end{tabular}
\end{table}

At the local scale with the trustworthiness statistic ($T(5)$), the Gini MDS ranks first respectively 3, 3, and 2 times out of 16 when using 1, 2, and 3 embedding components. At the same local scale, but considering 10 neighbors ($NN(10)$), the Gini MDS ranks first 6 times with 1 component, 6 times with 2 components, and 8 times with 3 components.

Analyzing the embedding space $\tilde \X$ at a global level thanks to the Silhouette score reveals that the 3 Euclidean MDS together outperform Gini MDS. Over the first component the Silhouette of the Gini MDS is top-ranked on 7 datasets out of 16 (9 times for the other MDSs). The same results are recorded when the embedding space is composed of two components. With three components, the Gini MDS outperforms all other methods with top-ranked positions on 9 datasets out of 16.

\subsubsection{Experiments with outliers}

By definition, an outlier is an extreme value, whose frequency of occurrence is very low, implying heavy-tailed distributions. In practice, the deviation of a given outlier from the mean must be at least 2 times the standard deviation. Accordingly, in order to stress the data, 2\% of the instances (taken at random) are multiplied by 10 times the standard deviation of $\X$. Since the Gini-pseudo distance can take benefit from the couples ``rank-values'' robust to outliers, we can expect the Gini MDS to achieve a better ranking under this stress setting. The contamination process is repeated 100 times (to limit computation time), and the similarity metrics are averaged over these 100 iterations.\footnote{Around 7 hours on an RTX 8000 for Gini MDS and the 3 Euclidean ones.}

\begin{table}[H]
\setlength{\tabcolsep}{1.5pt}
\footnotesize
\caption{\# Rank 1 occurrences on 16 datasets: 2\% contamination}
\label{results-MDS-noise}
\begin{tabular}{lcccccccc}
\hline\hline
\multirow{2}{*}{\textbf{Components}}
 & \multicolumn{2}{c}{\textbf{Trust} $T(5)$ \ \ }
 & \multicolumn{2}{c}{\textbf{Neighb.} $NN(10)$}
 & \multicolumn{2}{c}{\textbf{Silhouette} $S$} \\
\cline{2-7}
 & \textbf{MDS} & \textbf{Gini}
 & \textbf{MDS} & \textbf{Gini}
 & \textbf{MDS} & \textbf{Gini} \\
\hline
Component 1 & 9  & 7 & 13 & 3 & 9 & 7 \\
Component 2 & 11 & 5 & 9  & 7 & 8 & 8 \\
Component 3 & 12 & 4 & 9  & 7 & 9 & 7 \\
\hline
\end{tabular}
\end{table}

The most important metric is $T(5)$ since it relies on a real comparison of the data $\X$ and the embeddings. As shown in Table \ref{results-MDS-noise}, the Gini MDS becomes more competitive. The trustworthiness metric now ranks the Gini MDS first for 7 datasets out of 16 over the first component. The neighbors metric indicates that the Gini MDS achieves a top ranking for 7 datasets over 2 and 3 components, with a decline over the first component. The Silhouette score ranking remains stable compared to the case without outlier.

Another experiment has been done with 5\% contamination. Again, $T(5)$ and $NN(10)$ are slightly improved over the first and the second component respectively, whereas it is not the case for the Silhouette score. This suggests that the Gini MDS better preserves the local structure of the data. 


\subsubsection{Computation time}

The \texttt{sklearn} library allows MDS to be run using precomputed distances. While this is convenient for testing the robustness of different distance metrics, the library is not optimized for large datasets. By contrast, our tensor-based pairwise computation of Gini and Euclidean distances implemented in \texttt{PyTorch} (with GPU) provides faster results. For illustration, we consider the Banknote dataset, the largest of the 16 UCI datasets, with 1,372 instances and 4 features. We iterate over 30 values of the Gini hyperparameter $\nu$ ranging from 1.1 to 5 (increment of 0.1). As shown in Table \ref{table:time}, our implementation is significantly faster, requiring about 9 seconds compared to 46 seconds with \texttt{sklearn}\footnote{9.07 seconds on GPU RTX 8000, see  \href{https://github.com/bangtan66708/MDS_Gini_pseudo_metric/tree/main}{GitHub}.} (over 3 components).

\begin{table}[H]
\centering
\footnotesize
\caption{Execution Times (Banknote: $n=1,372$)}\label{table:time}
\begin{tabular}{c c}
\hline\hline
\textbf{MDS Types} & \textbf{Time} \\
\hline
\texttt{Pytorch} Gini MDS ($\nu^*=3.70$) & 9.071s \\
\texttt{Pytorch} Euclidean MDS      & 0.320s \\
\texttt{sklearn} Euclidean MDS   & 46.013s \\
\hline
\end{tabular}
\end{table}

\subsection{MNIST data}\label{MNIST}

The experiment on MNIST data \cite{Lecun1998} (images of size $28 \times 28$ pixels \textit{i.e.} 784 dimensions) is well suited, because the Gini pseudo-distance is based on ranks across dimensions rather than ranks across observations, in contrast to standard Gini metrics \cite{Yitzhaki2013,Sang2024,Mussard2025}. Since in a high-dimensional setting the Gini pseudo-distance exhibits more rank gaps, this could improve its ability to attenuate the influence of noisy points. We show that reconstructing the original images from the embeddings obtained with Gini MDS yields better classifications (with 1 component and 9 classifiers).

\subsubsection{Classification without noise}

In tasks of MNIST classification, it is usually hard for classifiers to distinguish between 5 and 6 (Figure \ref{fig:mnist-sample}).   

\begin{figure}[h!]
    \centering
    \includegraphics[width=5cm]{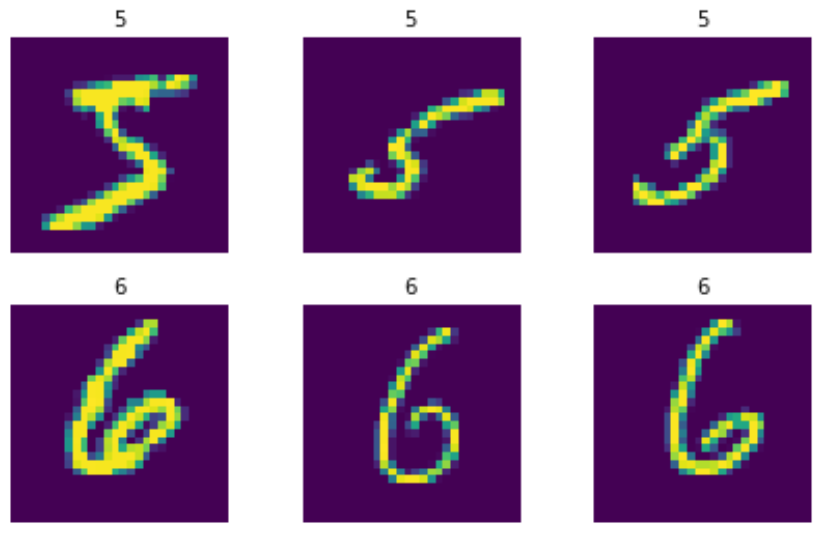}
    \caption{MNIST digits}
    \label{fig:mnist-sample}
\end{figure}

For this task of classification, we proceed in three steps. 

$\bulle$ \underline{Step 1:} 2,500 training images of each label (5 and 6) are taken at random to produce MDS embeddings (optimized Gini and standard Euclidean). 

$\bulle$ \underline{Step 2:} The Nyström technique is applied \cite{Williams2001nystrom} to reconstruct images. Images are first embedded onto one component, then a kernel (Gram) matrix is used to learn an inverse mapping from the embedded space back to the original input space. This allows the 2,500 images of the test data to be reconstructed from 1 component only.

$\bulle$ \underline{Step 3:}
To avoid the dependence to a particular classifier, 9 classifiers are employed to predict whether reconstructed images from 1 component are of label 5 or 6 on a test dataset of 250 images for each label and each fold (K-folds = 10): the multi-layer perceptron, the logistic regression, the k-nearest neighbors (with 5 neighbors), the linear and quadratic discriminant analyses, the naive Bayes classifier, the ridge classifier, and the linear SVM (for faster computation). A hard voting scheme is applied: for each reconstructed image in the test set, each classifier predicts a class, and the final assigned label is determined by the majority voting.

$\bulle$ \underline{Step 4:} Steps 1 to 3 are repeated for a 10-fold validation.

The results with one component are depicted in  Table \ref{tab:mds_results}. As expected Gini and Euclidean MDS are very close, with slightly better performance for the Euclidean: a F-measure around 0.9 for the Euclidean MDS against 0.89 for Gini.

\begin{table}[H]
\centering
\footnotesize
\caption{Performance of the MDS (1 Component)}
\begin{tabular}{ccccc}
\hline\hline
\textbf{MDS Type} & \textbf{Class} & \textbf{Precision} & \textbf{Recall} & \textbf{F1-Score} \\
\hline
\multirow{2}{*}{Euclidean MDS} & 5 & 0.8760 & 0.9268 & 0.9007 \\
                               & 6 & 0.9223 & 0.8688 & 0.8947 \\

\multirow{2}{*}{Gini MDS}      & 5 & 0.8680 & 0.9208 & 0.8936 \\
                               & 6 & 0.9157 & 0.8600 & 0.8870 \\
\hline
\end{tabular}
\label{tab:mds_results}
\end{table}

With 2 components, the Euclidean MDS reports a global F-measure (macro-average) of 0.9180 against 0.9074 for the Gini. With 3 components, the global F-measures are almost equivalent: 0.9458 for the standard MDS, 0.9410 for Gini. 

\subsubsection{Classification with noise}

Compared to Section \ref{MNIST} focused on outliers, we now propose an experiment of classification with noise. The difference is that all pixels are contaminated by an independent Gaussian noise drawn from a standard normal distribution. However, the level of contamination is very low, since to each pixel value is added the Gaussian noise scaled by one standard deviation of $\X$ (Figure \ref{fig:mnist-sample-2}). 

\begin{figure}[h!]
    \centering
    \includegraphics[width=6cm]{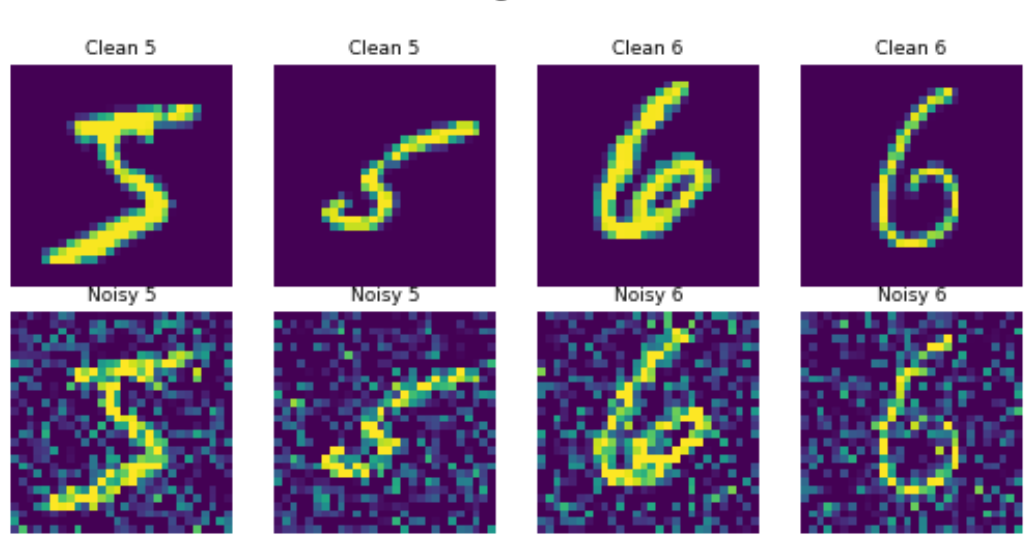}
    \caption{MNIST digits with noise}
    \label{fig:mnist-sample-2}
\end{figure}

We proceed as before with 4 steps on noisy images. The voting classifier achieves lower performance. Over 1 component, the Euclidean MDS underperforms compared to Gini as shown in Table \ref{result:final-euc}.

\begin{table}[H]
\centering
\footnotesize
\caption{MDS performance on noisy MNIST (1 Component)}
\begin{tabular}{ccccc}
\hline\hline
\textbf{MDS Type} & \textbf{Class} & \textbf{Precision} & \textbf{Recall} & \textbf{F1-Score} \\
\hline
\multirow{2}{*}{Euclidean MDS} & 5 & 0.7591 & 0.8128 & 0.7850 \\
                               & 6 & 0.7985 & 0.7420 & 0.7692 \\

\multirow{2}{*}{Gini MDS}      & 5 & 0.7758 & 0.8276 & \textcolor{blue}{\textbf{0.8009}} \\
                               & 6 & 0.8153 & 0.7608 & \textcolor{blue}{\textbf{0.7871}} \\
\hline
\end{tabular}
\label{result:final-euc}
\end{table}

It is noteworthy that the metric used to optimize the MDS is the standard KS metric. Indeed, the robust Huber technique provides worst results \textit{i.e.}, a global F-measure of 0.4048 in the Euclidean case. Over 2 components the global F-measures are almost identical: 0.8660 for the Euclidean MDS and 0.8659 for Gini MDS. The same result holds true for 3 components: 0.8768 for the Euclidean MDS, 0.8784 for Gini MDS. 

This shows that the Gini MDS has more compression power since it captures more information on the first component compared to the Euclidean MDS when data are contaminated.

\subsection{Simulations: Gini MDS vs. nonlinear MDS}\label{sec:simul}

The Gini MDS, based on a nonlinear distance, must be compared to other dimensionality reduction methods being nonlinear: t-SNE, Isomap, and UMAP. 

t-SNE \cite{Maaten2008tsne} is a probabilistic method that preserves local neighborhood structures by minimizing the divergence between pairwise similarities in high- and low-dimensional spaces. It is relevant for visualization, but it may fail to preserve the global geometry. Isomap \cite{tenenbaum2000isomap} extends classical MDS by replacing Euclidean distances with geodesic distances computed on a neighborhood graph. It preserves the global manifold structure under the assumption of isometry, however it may be sensitive to noise. Finally, UMAP \cite{mcinnes2018umap} relies on a topological framework based on fuzzy sets to preserve both local and global structures. It is computationally efficient and scalable, but depends on several hyperparameters.

We simulate 6 heavy-tailed distributions ($d = 6$) with $n = 500$ instances:

$\bulle$ Gaussian $\mathcal{N}(0,1)$ with 5\% outliers drawn from $\mathcal{N}(0,10^2)$,

$\bulle$ Cauchy distribution $\mathcal{C}(0,1)$,

$\bulle$ Weibull distribution with shape parameter $k = 0.5$,

$\bulle$ Pareto distribution with shape parameter $\alpha = 2$,

$\bulle$ Student-$t$ distribution with $\nu = 2$ degrees of freedom,

$\bulle$ $\ell$og-$\mathcal{N}(0,1.5^2)$.


\begin{algorithm}[h!]
\caption{Experiment for heavy-tailed distributions}
\begin{algorithmic}[1]

\FOR{$k = 1$ to $500$}

    \STATE Generate dataset $\X^{(k)} \in \mathbb{R}^{500 \times 6}$ with 6 independently sampled heavy-tailed or contaminated features

    \STATE Initialize $\nu = 2$
    \FOR{$t = 1$ to $T=3$}
        \STATE \textbf{Step 1:} Gini MDS:  
        $\min \mathcal{L}_{\text{Sammon}} = \boldsymbol{\tilde X}^{(k)}$ $(\nu=2)$
        \STATE \textbf{Step 2:} Update 
        $
        \nu^{*} = \arg\min_{\nu}$ KS($\nu$)
        \STATE \textbf{Step 3:} $\nu \leftarrow \nu^{*}$
    \ENDFOR
    \STATE \hspace{0.5cm} $\boldsymbol{\tilde X}^{(k)}[:, 0:2]\leftarrow$ t-SNE with optimized perplexity
    \STATE \hspace{0.5cm} $\boldsymbol{\tilde X}^{(k)}[:, 0:2]\leftarrow$ Isomap with optimized neighbors
    \STATE \hspace{0.5cm} $\boldsymbol{\tilde X}^{(k)}[:, 0:2]\leftarrow$ UMAP with optimized neighbors
    \STATE Evaluation metrics: Trust, NN, Silhouette
    \STATE Correlation metrics: Pearson and Spearman (rank)
\ENDFOR
\STATE Compute final results by averaging metrics over 500
\end{algorithmic}
\end{algorithm}

The data are centered to zero median and unit variance. In this case, the Gini metric space $(\mathbb G^d,\mathcal D_G)$ (with $\bar{x}$ being the median) prevents the generalized Gini distance to disproportionately take benefit from scale effects. A binary label $y$ is defined based on the median of the first feature to enable evaluation using classification-oriented metrics.

\begin{figure}[h!]
    \centering
\includegraphics[width=1.05\linewidth]{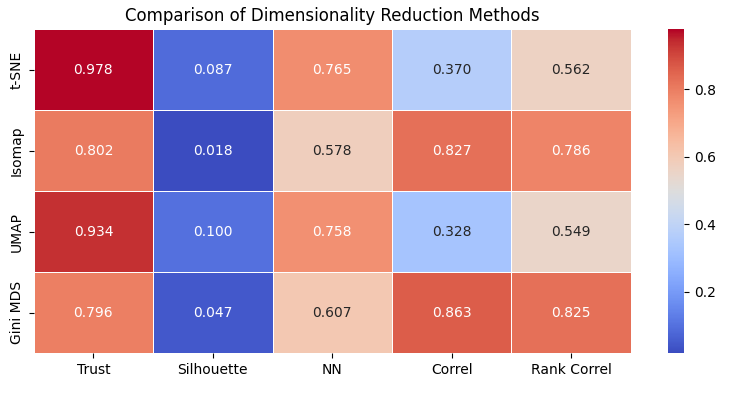}
    \caption{Simulations: nonlinear MDSs}
    \label{fig:simul}
\end{figure}

In total, 124,750 Euclidean distances are computed from the original data $\boldsymbol{X}^{(k)} \in \mathbb{R}^{500 \times 6}$, and likewise from the corresponding embeddings over two components $\boldsymbol{\tilde{X}}^{(k)} \in \mathbb{R}^{500 \times 2}$ (for the 4 methods). Pearson and Spearman (rank) correlations are computed between these two distance vectors, and then averaged over 500 replications. Although the optimized Gini MDS (last row of Figure~\ref{fig:simul}) ranks respectively 4th, 3rd, and 3rd in terms of Trustworthiness, Silhouette, and Nearest Neighbor metrics, these measures primarily capture local neighborhood preservation and clustering effects. In contrast, the Gini MDS achieves the highest correlation scores, with values of 0.863 and 0.825 (Pearson and Spearman, respectively), demonstrating its ability to preserve the global geometric structure of the data \textit{i.e.} the pairwise distance relationships.

\begin{table}[ht]
\centering
\caption{Execution and convergence time}
\label{tab:timing_methods}
\begin{tabular}{lcccc}
\hline\hline
\textbf{Nonlinear MDS} & \textbf{t-SNE} & \textbf{Isomap} & \textbf{UMAP} & \textbf{GiniMDS} \\
\hline
Time (average over 500 runs) & 3.34s & 0.15s & 1.1s & 2.0s \\
Convergence time (best run) & 1.10s & 0.15s & 0.35s & 0.80s \\
GPU / CPU & CPU & CPU & CPU & CPU \\
\bottomrule
\end{tabular}
\end{table}

The convergence time corresponds to the runtime of a single embedding with optimal hyperparameters, whereas the total time includes the full hyperparameter tuning procedure. All methods are executed on CPU for a fair comparison.

\section{Conclusion}\label{conslusion}

The optimized Gini multidimensional scaling is based on a generalized Gini (pseudo-)distance. It combines value and rank vectors with a tunable hyperparameter $\nu$, providing a robust framework for dimensionality reduction.

The empirical results, on 16 UCI datasets, show that Gini MDS performs better than the Euclidean MDS and its optimized variants (Huber, Sammon, SMACOF), especially under outlying contamination. While Euclidean methods perform well in clean settings for global structure, Gini MDS is able to preserve local structure, as shown by trustworthiness and neighborhood metrics. Applied to image data in a higher-dimensional setting, the Gini MDS combined with the Nyström method allows a robust reconstruction and classification of MNIST digits. In noisy conditions, it yields improved classification performance over Euclidean MDS, thanks to the rank difference captured by the Gini pseudo-distance. Finally, simulations on heavy-tailed distributions reveal that the Gini MDS outperforms other nonlinear methods maintaining pairwise distances in the embedding space. 



\bibliographystyle{IEEEtran}
\bibliography{IEEEexample}

\section{Appendix A: Proof}\label{appendix-A}

\noindent \textbf{Proof of Proposition 1:}
\noindent \noindent 
\newline $(\imath)$ It comes:
\begin{align}
     \|\lambda x \|_G &=  \mathcal{D}_G(\lambda x, 0_d) = \sum_{j=1}^d  \lambda x_j \left( R(\lambda x_{j}) - \frac{d+1}{2} \right) \notag
\end{align}
Since $R(\cdot)$ is an homogeneous function of degree zero if $\lambda \ge 0$, \textit{i.e.} $R(\lambda x) = R(x)$, then: 
\begin{align*}
    \|\lambda x \|_G &= \sum_{j=1}^d  \lambda x_j \left( R(x_{j}) - \frac{d+1}{2} \right) \\
    &= \lambda\left(\sum_{j=1}^d x_j \left( R(x_{j}) - \frac{d+1}{2} \right) \right) = \lambda \| x \|_G \\
\end{align*}
Since $R(-x)=d+1-R(x)$, if $\lambda < 0$ then: 
\begin{align}
     \|\lambda x \|_G &= \sum_{j=1}^d  \lambda x_j \left( R(\lambda x_{j}) - \frac{d+1}{2} \right) \\ 
     &= \sum_{j=1}^d  \lambda x_j \left( d+1 - R(x_{j}) - \frac{d+1}{2} \right) \notag \\
     &= -\lambda\left(\sum_{j=1}^d x_j \left( R(x_{j}) - \frac{d+1}{2} \right) \right) = | \lambda | \| x \|_G  \notag 
\end{align}
\newline $(\imath\imath)$ Let us now prove the triangle inequality: 
$ \lVert x + y\rVert_G \leq  
\lVert x \rVert_G +  
\lVert y \rVert_G$, $\forall  x,y \in \mathbb E^d$. Let $x,y \in \mathbb E^d$ such that $z = x + y$, we have:
\begin{align}
\lVert z \rVert_G &= \sum_{j=1}^d (x_j+y_j)\left[R(x_j+y_j)-\frac{d+1}{2}\right] \notag \\
&= \lVert x \rVert_G + \lVert y \rVert_G + \sum_{j=1}^d(R(z_j)-R(x_j))x_j \notag \\
&\quad + \sum_{j=1}^d(R(z_j)-R(y_j))y_j \notag
\end{align}
In order to get $ \lVert x + y\rVert_G \leq  
\lVert x \rVert_G +  \lVert y \rVert_G$, sufficient conditions are $\sum_{j=1}^d(R(z_j)-R(x_j))x_j \leq 0$  and $\sum_{j=1}^d(R(z_j)-R(y_j))y_j \leq 0$. To show $\sum_{j=1}^d(R(z_j)-R(x_j))x_j \leq 0$, one has to show that $\sum_{j=1}^d R(z_j) x_j \leq \sum_{j=1}^d R(x_j) x_j$. Let $\tilde x$ be such that $\tilde x_1 \leq \cdots \leq \tilde x_d$. From \cite{HLP1952}, for all permutations $R^\pi(x)$ of the elements of vector $R(x)$, $\sum_{j=1}^d R^\pi(x_j) \tilde x_j \leq \sum_{j=1}^d R(\tilde x_j) \tilde x_j = \sum_{j=1}^d R(x_j) x_j$. Therefore, $ \sum_{j=1}^d R(\tilde x_j)\tilde x_j$ is the maximum possible summation according to all possible permutations, so that $\sum_{j=1}^d R(z_j)x_j \leq \sum_{j=1}^d R(x_j)x_j$. The same reasoning applies to get $\sum_{j=1}^d R(z_j)y_j \leq \sum_{j=1}^d R(y_j)y_j$, and this ends the proof. 

\medskip

\noindent \textbf{Proof of Proposition 2 :}

\noindent $(\imath)$ $\mathcal{D}_G(x,y) = 0$ if $x=y$ and $(\imath\imath)$ $\mathcal{D}_G(x,y) = 0$ if $x=c\mathds{1}_d = \lambda y$, $\forall c \neq \lambda \in \R \setminus \{0\}$ are left to the reader. $(\imath\imath\imath)$ Let us prove symmetry by noting that $R(x_j-y_j) = d+1 - R(y_j - x_j)$:
\begin{align}
    \mathcal{D}_G(x,y) &= \sum_{j=1}^d (x_j - y_j)\Big(R(x_{j} - y_{j}) - \frac{d+1}{2}\Big) \notag \\
    &= \sum_{j=1}^d (y_j - x_j)(- R(x_{j} - y_{j}) + \frac{d+1}{2}) \notag \\
    &= \sum_{j=1}^d (y_j - x_j)\Big(R(y_{j} - x_{j}) - d - 1  +\frac{d+1}{2}\Big) 
    \notag \\
    &= \sum_{j=1}^d (y_j - x_j)\Big(R(y_{j} - x_{j}) -\frac{d+1}{2}\Big) 
    = \mathcal{D}_G(y,x) \notag
\end{align}
$(\imath\imath\imath)$ From the triangle inequality of $\lVert x \rVert_G $, 
let us prove that for all $x,y,z \in \mathbb E^d, \mathcal{D}_G(x,z) \leq \mathcal{D}_G(x,y) + \mathcal{D}_G(y,z)$.
From the triangle inequality of the Gini seminorm, we can write for all $u,v\in \mathbb E^d$:
$ \lVert u + v \rVert_G \leq  
\lVert u \rVert_G +  
\lVert v \rVert_G$. 
Setting $u = y-x$ and $v = z-y$ the inequality rewrites:
\begin{align}
     \lVert y-x + z-y \rVert_G &\leq 
    \lVert y-x \rVert_G +  
    \lVert z-y \rVert_G \notag \\
     \lVert z-x \rVert_G &\leq  
    \lVert y-x \rVert_G +  
    \lVert z-y \rVert_G \notag
\end{align}
By definition of the Gini seminorm and by symmetry of the distance, $\lVert z-x \rVert_G = \mathcal{D}_G(z,x) = \mathcal{D}_G(x,z)$, $\lVert y-x \rVert_G = \mathcal{D}_G(y,x)= \mathcal{D}_G(x,y)$ and $\lVert z-y \rVert_G = \mathcal{D}_G(z,y)= \mathcal{D}_G(y,z)$. Then, $\mathcal{D}_G(x,z) \leq \mathcal{D}_G(x,y) + \mathcal{D}_G(y,z)$, and this ends the proof.

\noindent $(\imath v)$ Finally, non-negativity can be shown to be true. From the triangle inequality, it has been proven that for all $x,y,z \in \mathbb E^d, \mathcal{D}_G(x,y) \leq \mathcal{D}_G(x,z)+\mathcal{D}_G(z,y)$. 
Setting $x=y$, the triangle inequality becomes $\mathcal{D}_G(x,x) \leq 2 \mathcal{D}_G(x,z)$. From the Null property, $\mathcal{D}_G(x,x)=0$. The triangle inequality rewrites $0 \leq \mathcal{D}_G(x,z)$ for all $x,z \in \mathbb E^d$, and this ends the proof. 

\noindent $(v)$ The Gini distance is actually a pseudo-distance since $\mathcal{D}_G(x,y) = 0$ is possible when $x \neq y$. This brings out an `ìf'' condition for $x = y$ and not an equivalence that arises only for egalitarian distributions. Given the pseudo-metric space $(\mathbb E^d, \mathcal{D}_G)$, we define the equivalence relation $x \sim y \iff \mathcal{D}_G(x,y) = 0$, that is, $x \sim y$ covers Null ($\imath$) and Egalitarian distributions
($\imath\imath$). Then, the corresponding quotient space is $\mathbb E^d / \sim$. 
The induced metric $\mathcal{D}_G$ on the quotient space $\mathbb E^d / \sim$ is a distance, so that $(\mathbb E^d / \sim, \mathcal{D}_G)$ is a Gini metric space. 

\medskip

\noindent \textbf{Proof of Proposition 3 :}

\noindent Let us define the norm $ \lVert x \rVert_{G,\nu} := \mathcal{D}_{G,\nu}(x,0)$. Let us prove the triangle inequality 
$ \lVert x + y\rVert_{G,\nu} \leq  
\lVert x \rVert_{G,\nu} +  
\lVert y \rVert_{G,\nu}$, for all $x,y \in \mathbb E^d$. Let $z = x + y$, we have:
\begin{align*}
\lVert z \rVert_{G,\nu}  &= -d\sum_{j=1}^d (x_j+y_j)\left[\overline{F}^{\nu-1}(z_j)-\overline{F}^{\nu-1}(0)\right] \notag \\
&= \lVert x \rVert_{G,\nu} + \lVert y \rVert_{G,\nu} -d \sum_{j=1}^d x_j (\overline{F}^{\nu-1}(z_j)-\overline{F}^{\nu-1}(x_j))\\& -d \sum_{j=1}^d y_j(\overline{F}^{\nu-1}(z_j)-\overline{F}^{\nu-1}(y_j)) 
\end{align*}

In the same way as in Proposition 1, using the inequality of Hardy, Littlewood and Pólya (1952) yields \\
$-d \sum_{j=1}^d x_j (\overline{F}^{\nu-1}(z_j)-\overline{F}^{\nu-1}(x_j)) \leq 0$ and $-d \sum_{j=1}^d y_j(\overline{F}^{\nu-1}(z_j)-\overline{F}^{\nu-1}(y_j))\leq 0$, therefore, $ \lVert z\rVert_{G,\nu} = \lVert x+y \rVert_{G,\nu} \leq  
\lVert x \rVert_{G,\nu} +  
\lVert y \rVert_{G,\nu}$.

Now, we define a new norm associated to the pseudo-distance $\overline{\mathcal{D}}_{G,\nu}$. Let the symmetric norm associated to the distance $\overline{\mathcal{D}}_{G,\nu}$ be defined as $\lVert x - y\rVert_{G,\nu}^s := \mathcal{\overline{D}}_{G,\nu}(x,y) = \frac{1}{2}\lVert x - y\rVert_{G,\nu} + \frac{1}{2}\lVert y - x\rVert_{G,\nu} = \frac{1}{2}\mathcal{D}_{G,\nu}(x,y) + \frac{1}{2}\mathcal{D}_{G,\nu}(y,x)$, for all $x,y \in \mathbb E^d$. Let us prove the triangle inequality
$ \lVert x + y\rVert_{G,\nu}^s \leq  
\lVert x \rVert_{G,\nu}^s +  
\lVert y \rVert_{G,\nu}^s$ for all $x,y \in \mathbb E^d$. Since $\lVert \cdot \rVert_{G,\nu}$ respects the triangle inequality then:
\begin{align}
\lVert x + y\rVert_{G,\nu}^s &= \frac{1}{2}\lVert x + y\rVert_{G,\nu} + \frac{1}{2}\lVert -x - y\rVert_{G,\nu} \notag \\
&\leq \frac{1}{2}\Big[\lVert x \rVert_{G,\nu} + \lVert y\rVert_{G,\nu}\Big] + \frac{1}{2}\Big[\lVert -x \rVert_{G,\nu} + \lVert -y\rVert_{G,\nu}\Big] \notag \\
&= \frac{1}{2}\Big[\lVert x \rVert_{G,\nu} + \lVert -x\rVert_{G,\nu}\Big] + \frac{1}{2}\Big[\lVert y \rVert_{G,\nu} + \lVert -y\rVert_{G,\nu}\Big] \notag \\ 
&= \lVert x \rVert_{G,\nu}^s + \lVert  y\rVert_{G,\nu}^s \label{norm-2}
\end{align}
From the triangle inequality of the symmetric norm Eq.\eqref{norm-2}, 
let us prove that for all $x,y,z \in \mathbb E^d, \mathcal{\overline{D}}_{G,\nu}(x,z) \leq \mathcal{\overline{D}}_{G,\nu}(x,y) + \mathcal{\overline{D}}_{G,\nu}(y,z)$.
From the triangle inequality, for all $u,v\in \mathbb E^d$:
$ \lVert u + v \rVert_{G,\nu}^s \leq  
\lVert u \rVert_{G,\nu}^s +  
\lVert v \rVert_{G,\nu}^s$. 
Setting $u = y-x$ and $v = z-y$ the inequality rewrites:
\begin{align}
     \lVert y-x + z-y \rVert_{G,\nu}^s &\leq 
    \lVert y-x \rVert_{G,\nu}^s +  
    \lVert z-y \rVert_{G,\nu}^s \notag \\
     \lVert z-x \rVert_{G,\nu}^s &\leq  
    \lVert y-x \rVert_{G,\nu}^s +  
    \lVert z-y \rVert_{G,\nu}^s \notag
\end{align}
Since $\lVert z-x \rVert_{G,\nu}^s = \mathcal{\overline{D}}_{G,\nu}(z,x) = \mathcal{\overline{D}}_{G,\nu}(x,z)$, $\lVert y-x \rVert_{G,\nu}^s = \mathcal{\overline{D}}_{G,\nu}(y,x)= \mathcal{\overline{D}}_{G,\nu}(x,y)$ and $\lVert z-y \rVert_{G,\nu}^s = \mathcal{\overline{D}}_{G,\nu}(z,y)= \mathcal{\overline{D}}_{G,\nu}(y,z)$. Then, $\mathcal{\overline{D}}_{G,\nu}(x,z) \leq \mathcal{\overline{D}}_{G,\nu}(x,y) + \mathcal{\overline{D}}_{G,\nu}(y,z)$.

From the triangle inequality of the distance $\mathcal{\overline{D}}_{G,\nu}$, we have: 
$\mathcal{\overline{D}}_{G,\nu}(x,y) \leq \mathcal{\overline{D}}_{G,\nu}(x,z)+\mathcal{\overline{D}}_{G,\nu}(z,y)$. Setting $x=y$, the triangle inequality becomes $\mathcal{\overline{D}}_{G,\nu}(x,x) \leq 2 \mathcal{\overline{D}}_{G,\nu}(x,z)$. Note that $\mathcal{\overline{D}}_{G,\nu}$ satisfies the Null property since $\mathcal{\overline{D}}_{G,\nu}(x,x)=\frac{1}{2}\mathcal{D}_{G,\nu}(x,x)+\frac{1}{2}\mathcal{D}_{G,\nu}(x,x) = 0$. Therefore we have $0 \leq \mathcal{\overline{D}}_{G,\nu}(x,z)$ for all $x,z \in \mathbb E^d$.

\end{document}